\documentclass[11pt]{article}
\usepackage[numbers,compress]{natbib}
\usepackage{authblk}
\usepackage[lmargin=1in,rmargin=1in, bmargin=1.25in, tmargin=1.25in]{geometry}

\usepackage{amsmath,amsthm,amssymb,amsfonts}
\usepackage{tikz}
\usepackage{hyperref}
\usepackage{color}
\usepackage{xcolor}
\usepackage{mathtools}
\usepackage{pgf}
\usepackage{mathrsfs}
\usetikzlibrary{arrows}
\usepackage{caption}
\usepackage{subcaption}

\definecolor{qqqqcc}{rgb}{0.,0.,0.8}
\definecolor{zzqqtt}{rgb}{0.6,0.,0.2}
\definecolor{eqbqff}{rgb}{0.8784313725490196,0.6901960784313725,1.}
\definecolor{ududff}{rgb}{0.30196078431372547,0.30196078431372547,1.}

\newcommand{\R}{\mathbb{R}}
\newcommand{\reals}{\R}
\newcommand{\rationals}{\mathbb{Q}}
\newcommand{\N}{\mathbb{N}}

\newcommand{\abs}[1]{ \left\vert #1\right\vert }
\newcommand{\set}[1]{\left\{#1\right\}}

\newcommand{\eval}[2]{\underset{{#1}}{\mathbb{E}}\left[#2\right]}
\renewcommand{\Pr}[2]{\underset{{#1}}{\mathbb{P}}\left(#2\right)}

\newcommand{\roblossc}{\mathsf{R}^C_\rho}
\newcommand{\roblosse}{\mathsf{R}^E_\rho}
\newcommand{\boolhc}{\set{0,1}^n}
\newcommand{\boolhcprime}{\set{0,1}^{(2k+1)n+1}}

\newcommand{\poly}{\text{poly}}

\newcommand{\supp}{\text{supp}}
\newcommand{\maj}{\text{maj}}
\newcommand{\given}{\;|\;}

\newcommand{\A}{\mathcal{A}}
\newcommand{\C}{\mathcal{C}}
\newcommand{\Cprime}{\mathcal{C}'_{(k,n)}}
\newcommand{\D}{\mathcal{D}}
\newcommand{\Dprime}{\mathcal{D}'_{(k,n)}}
\newcommand{\E}{\mathcal{E}}
\newcommand{\F}{\mathcal{F}}

\renewcommand{\H}{\mathcal{H}}
\newcommand{\X}{\mathcal{X}}
\newcommand{\Xprime}{\X'_{k,n}}
\newcommand{\Y}{\mathcal{Y}}

\newcommand{\MonConj}{\textsf{MON-CONJ }}

\newtheorem{theorem}{Theorem}
\newtheorem{proposition}[theorem]{Proposition}
\newtheorem{lemma}[theorem]{Lemma}
\newtheorem{corollary}[theorem]{Corollary}

\newtheorem{definition}[theorem]{Definition}

\theoremstyle{remark}
\newtheorem{remark}[theorem]{Remark}

\title{On the Hardness of Robust Classification}

\author{Pascale Gourdeau}
\author{Varun Kanade}
\author{Marta Kwiatkowska}
\author{James Worrell}

\affil{University of Oxford}

\begin{document}
\maketitle
\begin{abstract} 

It is becoming increasingly important to understand the vulnerability
of machine learning models to adversarial attacks.  In this paper we
study the feasibility of robust learning from the perspective of
computational learning theory, considering both sample and
computational complexity.  In particular, our definition of robust
learnability requires polynomial sample complexity.  We start with two
negative results.  We show that no non-trivial concept class can be
robustly learned in the distribution-free setting against an adversary
who can perturb just a single input bit.  We show moreover that the class of monotone
conjunctions cannot be robustly learned under the uniform distribution against an adversary who can
perturb $\omega(\log n)$ input bits.  However if the adversary is
restricted to perturbing $O(\log n)$ bits, then the class of monotone
conjunctions can be robustly learned with respect to a general class of
distributions (that includes the uniform distribution).  Finally, we
provide a simple proof of the computational hardness of robust
learning on the boolean hypercube.  Unlike previous results of this nature, our result does
not rely on another computational model (e.g. the statistical query
model) nor on any hardness assumption
other than the existence of a hard learning problem in the PAC
framework.
\end{abstract}

\section{Introduction}

There has been considerable interest in adversarial machine learning
since the seminal work of~\citet{szegedy2013intriguing}, who coined
the term \emph{adversarial example} to denote the result of applying a
carefully chosen perturbation that causes a classification error to a
previously correctly classified datum.  \citet{biggio2013evasion}
independently observed this phenomenon. However, as pointed out 
by \citet{biggio2017wild}, adversarial machine learning has been 
considered much earlier in the context of spam filtering~\cite{dalvi2004adversarial,lowd2005adversarial,lowd2005good}.
Their survey also distinguished two settings: \emph{evasion attacks}, 
where an adversary modifies data at test time, and \emph{poisoning attacks}, 
where the adversary modifies the training data.\footnote{\noindent For an in-depth review and definitions of different types of attacks, the reader may refer to \cite{biggio2017wild,dreossi2019formalization}.}

Several different definitions of adversarial learning exist in the
literature and, unfortunately, in some instances the same terminology
has been used to refer to different notions (for some discussion see
e.g., \cite{dreossi2019formalization,diochnos2018adversarial}). Our
goal in this paper is to take the most widely-used definitions and
consider their implications for robust learning from a statistical and
computational viewpoint. For simplicity, we will focus on the setting
where the input space is the boolean hypercube $\mathcal X = \{0,
1\}^n$ and consider the \emph{realizable} setting, i.e. the labels are
consistent with a target concept in some concept class.

An \emph{adversarial example} is constructed from a \emph{natural
  example} by adding a perturbation. Typically, the power of the
adversary is curtailed by specifying an upper bound on the
perturbation under some norm; in our case, the only meaningful norm is
the Hamming distance. For a point $x \in \mathcal X$, let $B_\rho(x)$
denote the Hamming ball of radius $\rho$ around $x$.  Given a
distribution $D$ on $\mathcal X$, we consider the \emph{adversarial
  risk} of a hypothesis $h$ with respect to a target concept $c$ and
perturbation budget $\rho$.  We focus on two definitions of risk.  The 
\emph{exact in the ball} risk $\roblosse(h,c)$ is the probability $\Pr{x\sim D}{\exists y \in B_\rho(x)
  \cdot h(y)\neq c(y)}$ that the adversary can perturb a point $x$
drawn from distribution $D$ to a point $y$ such that $h(y)\neq c(y)$.
The \emph{constant in the ball} risk $\roblossc(h,c)$ is the probability $\Pr{x\sim D}{\exists
  y \in B_\rho(x) \cdot h(y)\neq c(x)}$ that the adversary can perturb
a point $x$ drawn from distribution $D$ to a point $y$ such that
$h(y)\neq c(x)$. These definitions encode two different
interpretations of robustness.  In the first view, robustness speaks
about the fidelity of the hypothesis to the target concept, whereas in
the latter view robustness concerns the sensitivity of the output of
the hypothesis to corruptions of the input.  In
fact, the latter view of
robustness can in some circumstances be in conflict with accuracy in the
traditional sense~\cite{tsipras2019robustness}.

\subsection{Overview of Our Contributions}

We view our conceptual contributions to be at least as important as the technical results and believe that the issues highlighted in our work will result in more concrete theoretical frameworks being developed to study adversarial learning. 

\begin{figure}
	\begin{subfigure}{0.3\textwidth}
  		\centering
  			\includegraphics[width=\textwidth]{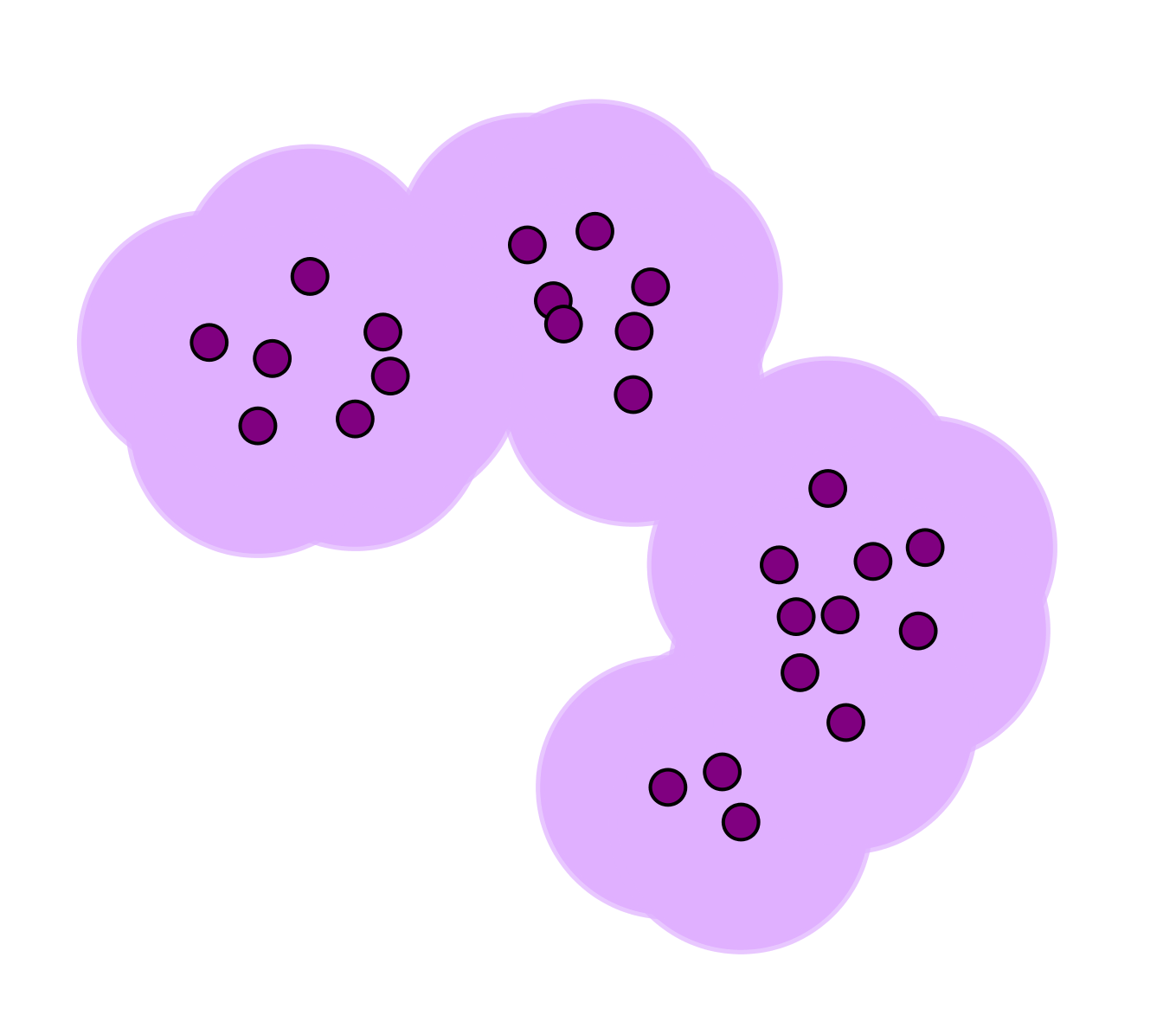}
  			\caption{}
  			\label{fig:constant}
	\end{subfigure}
	\hfill
	\begin{subfigure}{0.3\textwidth}
		\includegraphics[width=\textwidth]{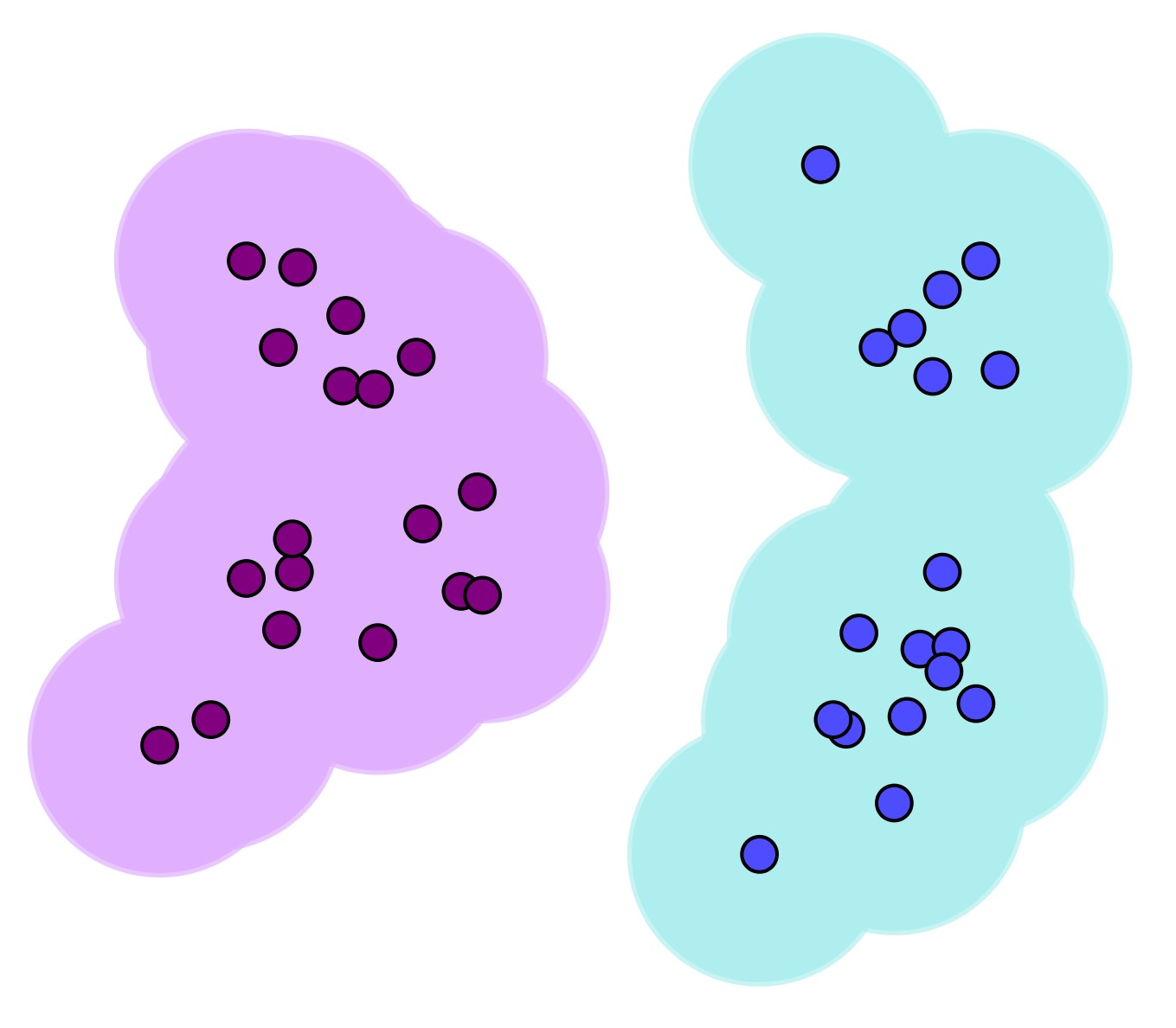}
  		\caption{}
		\label{fig:sep-supp}
	\end{subfigure}
	\hfill
	\begin{subfigure}{0.3\textwidth}
 	 \centering
		\includegraphics[width=\textwidth]{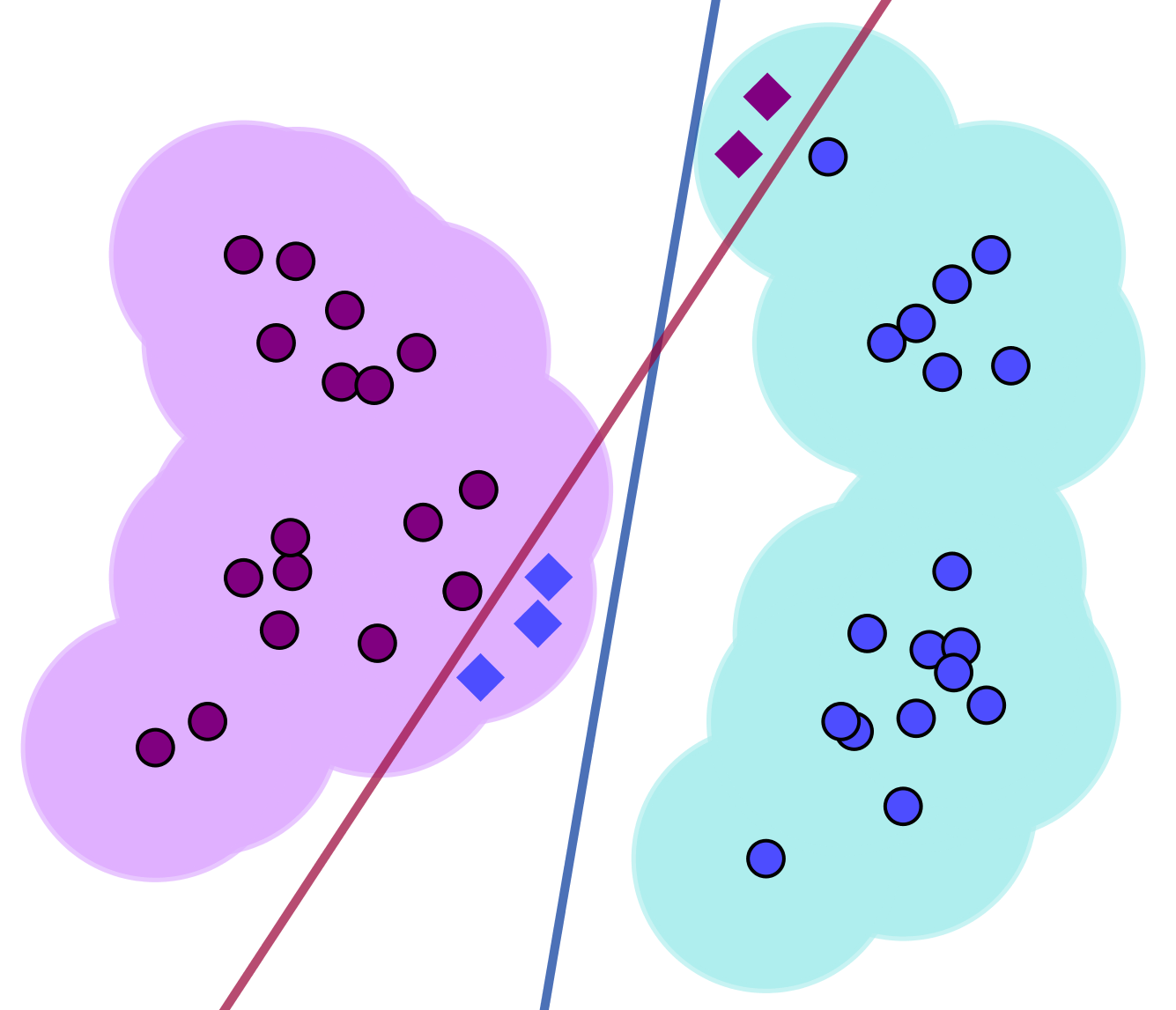}
		\caption{}
		\label{fig:diff-losses}
	\end{subfigure}
\caption{(a) The support of the distribution is such that $\roblossc(h,c)=0$ can only be achieved if $c$ is constant.
(b) The $\rho$-expansion of the support of the distribution and target $c$ admit hypotheses $h$ such that $\roblossc(h,c)=0$.
(c) An example where $\roblossc$ and $\roblosse$ differ. The red concept is the target, while the blue one is the hypothesis. The dots are the support of the distribution and the shaded regions represent their $\rho$-expansion. The diamonds represent perturbed inputs which cause $\roblosse>0$.
}
\label{fig:rob-losses}
\end{figure}



\subsubsection*{Impossibility of Robust Learning in Distribution-Free PAC Setting} 

We first consider the question of whether achieving \emph{zero} (or low) robust
risk is possible under either of the two definitions. If the \emph{balls} of
radius $\rho$ around the data points intersect so that the total region is
connected, then unless the target function is constant, it is impossible to
achieve $\roblossc(h, c)=0$ (see Figure~\ref{fig:rob-losses}). In particular,
in most cases $\roblossc(c, c) \neq 0$, i.e., even the target concept does not
have zero risk with respect to itself. We show that this is the case for
extremely simple concept classes such as \emph{dictators} or \emph{parities}.
When considering the \emph{exact on the ball} notion of robust learning, we at
least have $\roblosse(c, c) = 0$; in particular, any concept class that can be
exactly learned can be robustly learned in this sense. However, even in this
case we show that no ``non-trivial'' class of functions can be robustly
learned. We highlight that these results show that a polynomial-size sample
from the unknown distribution is not sufficient, even if the learning algorithm
has arbitrary computational power (in the sense of Turing computability).%
\footnote{We do require any operation performed by the learning algorithm is computable; the results of \citet{bubeck2018adversarial} imply that an algorithm that can potentially evaluate \emph{uncomputable} functions can always robustly learn using a polynomial-size sample. See the discussion on computational hardness below.}

\subsubsection*{Robust Learning of Monotone Conjunctions} 

Given the impossibility of distribution-free robust learning, we consider
robust learning under specific distributions. We consider one of the simplest
concept class studied in PAC Learning, the class of \emph{monotone
conjunctions}, under the class of $\log$-Lipschitz distributions (which
includes the uniform distribution) and show that this class of functions is
robustly learnable provided $\rho = O(\log n)$ and is not robustly learnable
with polynomial sample complexity for $\rho = \omega (\log n)$. A class of
distributions is said to be $\alpha$-$\log$-Lipschitz if the logarithm of the
density function is $\log(\alpha)$-Lipschitz with respect to the Hamming
distance.  Our results apply in the setting where the learning algorithm only
receives random labeled examples. On the other hand, a more powerful learning
algorithm that has access to membership queries can exactly learn monotone
conjunctions and as a result can also robustly learn with respect to
\emph{exact in the ball} loss.

\subsubsection*{Computational Hardness of PAC Learning} 

Finally, we consider computational aspects of robust learning. Our focus is on
two questions: \emph{computability} and \emph{computational complexity}. Recent
work by \citet{bubeck2018adversarial} provides a result that states that
minimizing the robust loss on a polynomial-size sample suffices for robust
learning. However, because of the existential quantifier over the ball implicit
in the definition of the \emph{exact in the ball} loss, the empirical risk
cannot be \emph{computed} as this requires enumeration over the \emph{reals}. Even if one restricted attention to concepts defined over $\rationals^n$, computing the loss would be \emph{recursively enumerable}, but not \emph{recursive}. In the case of functions
defined over finite instance spaces, such as the boolean hypercube, the loss can be evaluated provided the learning algorithm has access to a membership query oracle; 
for the \emph{constant in the ball} loss membership queries are not required.
For functions defined on $\reals^n$ it is unclear how either loss function can
be evaluated even if the learner has access to membership queries, since in principle it requires enumerating over the reals. Under
strong assumptions of \emph{inductive bias} on the target and hypothesis
class, it may be possible to evaluate the loss functions; however this would
have to be handled on a case by case basis -- for example, properties of the
target and hypothesis, such as Lipschitzness or large margin, could be used to compute the
exact in the ball loss in finite time.

Second, we consider the computational complexity of robust learning.
\citet{bubeck2018cryptographic} and~\citet{degwekar2019computational}
have shown that there are concept classes that are hard to robustly
learn under cryptographic assumptions, even when robust learning is
information-theoretically feasible.  \citet{bubeck2018adversarial} 
establish super-polynomial lower bounds for robust learning in
the \emph{statistical query} framework. We give an arguably simpler
proof of hardness, based simply on the assumption that there exist
concept classes that are hard to PAC learn. In particular, our
reduction also implies that robust learning is hard even if the
learning algorithm is allowed membership queries, provided the concept
class that we reduce from is hard to learn using membership
queries. Since the existence of one-way functions implies the existence
of concept classes that are hard to PAC learn (with or without
membership queries), our result is also based on a slightly weaker
assumption than~\citet{bubeck2018adversarial}\footnote{It is believed that the existence of hard to PAC learn concept
  classes is not sufficient to construct one-way
  functions.~\cite{ABX:2008}.}.

\subsection{Related work on the Existence of Adversarial Examples}

There is a considerable body of work that studies the inevitability of
adversarial examples,
e.g.,~\cite{fawzi2016robustness,fawzi2018analysis,fawzi2018adversarial,gilmer2018adversarial,shafahi2018adversarial}.
These papers characterize robustness in the sense that a classifier's output
on a point should not change if a perturbation of a certain
magnitude is applied to it.  Among other things, these works study
geometrical characteristics of classifiers and statistical
characteristics of classification data that lead to adversarial
vulnerability.

Closer to the present paper
are~\cite{diochnos2018adversarial,mahloujifar2018can,mahloujifar2019curse},
which work the with exact-in-a-ball notion of robust risk.
In particular, \cite{diochnos2018adversarial} considers the robustness
of monotone conjunctions under the uniform distribution on the boolean
hypercube for this notion of risk (therein called the \emph{error
  region} risk).  However~\cite{diochnos2018adversarial} does not
address the sample and computational complexity of learning: their
results rather concern the ability of an adversary to magnify the
missclassification error of \emph{any} hypothesis with respect to
\emph{any} target function by perturbing the input.  For
example, they show that an adversary who can perturb $O(\sqrt{n})$
bits can increase the missclassification probability from $0.01$ to
$1/2$.  By contrast we show that a weaker adversary, who can perturb
only $\omega(\log n)$ bits, renders it impossible to learn monotone
conjunctions with polynomial sample complexity.  The main tool used
in~\cite{diochnos2018adversarial} is the isoperimetric inequality for
the Boolean hypercube, which gives lower bounds on the volume of the
expansions of arbitrary subsets.  On the other hand, we use the
probabilistic method to establish the existence of a single hard-to-learn 
target concept for any given algorithm with polynomial sample
complexity.

\section{Definition of Robust Learning}
\label{sec:rob-loss}

The notion of robustness can be accommodated within the
basic set-up of PAC learning by adapting the definition of risk
function.  In this section we review two of the main definitions of
\emph{robust risk} that have been used in the literature.  For
concreteness we consider an input space $\mathcal{X}=\{0,1\}^n$ with
metric $d:\mathcal{X}\times\mathcal{X}\rightarrow\mathbb{N}$, where
$d(x,y)$ is the Hamming distance of $x,y\in\mathcal{X}$.  Given
$x\in\mathcal{X}$, we write $B_\rho(x)$ for the ball $\{y\in
\mathcal{X} : d(x,y)\leq \rho\}$ with centre $x$ and radius $\rho\geq
0$.

The first definition of robust risk asks that the hypothesis be
exactly equal to the target concept in the ball $B_\rho(x)$ of
radius $\rho$ around a ``test point'' $x\in\mathcal{X}$:
\begin{definition}
\label{def:loss-correct}
Given respective hypothesis and target functions
$h,c:\mathcal{X}\rightarrow\{0,1\}$, distribution $D$ on
$\mathcal{X}$, and robustness parameter $\rho\geq 0$, we define the
``exact in the ball'' robust risk of $h$ with respect to $c$ to be
\begin{equation*}
\roblosse(h,c)=\Pr{x\sim D}{\exists z\in B_\rho(x):h(z)\neq c(z)}
\, .
\end{equation*}
\end{definition}

While this definition captures a natural notion of robustness, an
obvious disadvantage is that evaluating the risk function requires
 the learner to have knowledge of the target function outside of the
training set, e.g., through membership queries.  Nonetheless, 
by considering a learner who has
oracle access to the predicate ${\exists z\in B_\rho(x):h(z)\neq
  c(z)}$, we can use the exact-in-the-ball framework to analyse 
sample complexity and to prove strong lower bounds
on the computational complexity of robust learning.

A popular alternative to the exact-in-the-ball risk function in
Definition~\ref{def:loss-correct} is the following
\emph{constant-in-the-ball risk} function:
\begin{definition}
\label{def:loss-constant}
Given respective hypothesis and target functions
$h,c:\mathcal{X}\rightarrow\{0,1\}$, distribution $D$ on
$\mathcal{X}$, and robustness parameter $\rho\geq 0$, we define the
``constant in the ball'' robust risk of $h$ with respect to $c$ as
\begin{equation*}
\roblossc(h,c)=\Pr{x\sim D}{\exists z\in B_\rho(x):h(z)\neq c(x)}
\enspace.
\end{equation*}
\end{definition}

An obvious advantage of the constant in the ball risk over
the exact in the ball version is that in the former,
evaluating the loss at point $x\in\mathcal{X}$ requires only knowledge
of the correct label of $x$ and the hypothesis $h$.  In particular,
this definition can also be carried over to the non-realizable
setting, in which there is no target.  However, from a foundational
point of view the constant in the ball risk has some drawbacks: recall from 
the previous section that under this definition it is possible to have strictly
positive robust risk in the case that $h=c$.  (Let us note in passing that the
risk functions $\roblossc$ and $\roblosse$ are in general
incomparable.  Figure~\ref{fig:diff-losses} gives an example in which
$\roblossc=0$ and $\roblosse>0$.)  Additionally, when we work in the
hypercube, or a bounded input space, as $\rho$ becomes larger, we
eventually require the function to be constant in the whole space.
Essentially, to $\rho$-robustly learn in the realisable setting, we
require concept and distribution pairs to be represented as two sets
$D_+$ and $D_-$ whose $\rho$-expansions don't intersect, as
illustrated in Figures~\ref{fig:constant} and~\ref{fig:sep-supp}.  These limitations appear even more
stringent when we consider simple concept classes such as parity
functions, which are defined for an index set $I\subseteq [n]$ as
$f_I(x)=\sum_i x_i +b\mod 2$ for $b\in\{0,1\}$.  This class can be PAC-learned, as well as
exactly learned with $n$ membership queries.  However, for any point,
it suffices to flip one bit of the index set to switch the label, so
$\roblossc(f_I,f_I)=1$ for any $\rho\geq1$ if $I\neq\emptyset$. 

Ultimately, we want the adversary's power to come from creating
perturbations that cause the hypothesis and target functions to differ
in some regions of the input space.  For this reason we favor the
exact-in-the-ball definition and henceforth work with that.  

Having settled on a risk function, we now formulate the definition of
robust learning.  For our purposes a \emph{concept class} is a family
$\mathcal{C} = \{\mathcal{C}_n\}_{n\in \mathbb{N}}$, with
$\mathcal{C}_n$ a class of functions from $\{0,1\}^n$ to $\{0,1\}$.
Likewise a \emph{distribution class} is a family $\mathcal{D} = \{
\mathcal{D}_n\}_{n\in\mathbb{N}}$, with $\mathcal{D}_n$ a set of
distributions on $\{0,1\}^n$.  Finally a \emph{robustness function} is
a function $\rho:\mathbb{N}\rightarrow \mathbb{N}$.

\begin{definition}
\label{def:robust-learning}
Fix a function $\rho:\N\rightarrow\N$. We say that an algorithm $\A$
\emph{efficiently} $\rho$-\emph{robustly learns} a concept class $\C$
with respect to distribution class $\mathcal{D}$ if there exists a
polynomial $\poly(\cdot,\cdot,\cdot)$ such that for all
$n\in\mathbb{N}$, all target concepts $c\in \C_n$, all distributions
$D \in \mathcal{D}_n$, and all accuracy and confidence parameters
$\epsilon,\delta>0$, there exists $m \leq
\poly(1/\epsilon,1/\delta,n)$, such that when $\A$ is given access to
a sample $S\sim D^m$ it outputs $h:\{0,1\}^n\rightarrow\{0,1\}$ such
that $\Pr{S\sim D^m}{\mathsf{R}^E_{\rho(n)}(h,c)<\epsilon}>1-\delta$.
\end{definition}

Note that the definition of robust learning requires polynomial sample
complexity and allows improper learning (the hypothesis $h$ need not
belong to the concept class $\mathcal{C}_n$).

In the standard PAC framework, a hypothesis $h$ is considered to
have zero risk with respect to a target concept $c$ when $\Pr{x\sim
D}{h(x)\neq c(x)}=0$.  We have remarked that exact learnability
implies robust learnability; we next give an example of a
concept class $\C$ and distribution $D$ such that $\C$ is PAC learnable under $D$ with zero risk and yet cannot be
robustly learned under $D$ (regardless of the sample complexity).

\begin{lemma}
\label{lemma:dictators}
The class of  dictators is not 1-robustly learnable (and thus not robustly learnable for any $\rho\geq1$) with respect to the robust risk of Definition~\ref{def:loss-correct} in the distribution-free setting. 
\end{lemma}
\begin{proof}
Let $c_1$ and $c_2$ be the dictators on variables $x_1$ and $x_2$, respectively.
Let $D$ be such that $\Pr{x\sim D}{x_1=x_2}=1$ and $\Pr{x\sim D}{x_k=1}=\frac{1}{2}$
for $k\geq3$. 
Draw a sample $S\sim D^m$ and label it according to $c\sim U(c_1,c_2)$. 
By the choice of $D$, the elements of $S$ will have the same label regardless of whether $c_1$ or $c_2$ was picked.
However, for $x\sim D$, it suffices to flip any of the first two bits to cause $c_1$ and $c_2$ to disagree on the perturbed input.
We can easily show that, for any $h\in\set{0,1}^\X$, 
 $
 \mathsf{R}^E_1(c_1,h)+ \mathsf{R}^E_1(c_2,h)
\geq  \mathsf{R}^E_1(c_1,c_2) = 1.
$
Then 
\begin{equation*}
\underset{c\sim U(c_1,c_2)}{\mathbb{E}}\eval{S\sim D^m}{ \mathsf{R}^E_1(h,c)} \geq 1/2 \enspace.
\end{equation*}
We conclude that one of $c_1$ or $c_2$ has robust risk at least 1/2.
\end{proof}
Note that a PAC learning algorithm with error probability threshold $\varepsilon=1/3$ will either output  $c_1$ or $c_2$ and will hence have standard risk  zero. We refer the reader to Appendix~\ref{app:zero-risk-learning} for further discussion on the relationship between robust and zero-risk learning.



\section{No Distribution-Free Robust Learning in $\boolhc$}
\label{sec:no-df-rl}

In this section, we show that no non-trivial concept class is efficiently 1-robustly learnable in the boolean hypercube.
Such a class is thus not efficiently $\rho$-robustly learnable for any $\rho\geq1$. 
Efficient robust learnability then requires access to a more powerful learning model or distributional assumptions. 

Let $\C_n$ be a concept class on $\boolhc$, and define a concept class as $\C=\bigcup_{n\geq 1} \C_n$.
We say that a class of functions is trivial if $\C_n$ has at most two functions, and that they differ on every point.

\begin{theorem}
\label{thm:no-df-rl}
Any concept class $\C$ is efficiently distribution-free robustly learnable iff it is trivial.
\end{theorem}

The proof of the theorem relies on the following lemma:
\begin{lemma}
\label{lemma:robloss-triangle}
Let $c_1,c_2\in\{0,1\}^\X$ and fix a distribution on $\X$. 
Then for all $h:\boolhc\rightarrow\set{0,1}$
\begin{equation*}
\roblosse(c_1,c_2)
\leq \roblosse(c_1,h)+\roblosse(c_2,h)
\enspace.
\end{equation*}
\end{lemma}
\begin{proof}
Let $x\in\boolhc$ be arbitrary, and suppose that $c_1$ and $c_2$ differ on some $z\in B_\rho(x)$. 
Then either $h(z)\neq c_1(z)$ or $h(z)\neq c_2(z)$. The result follows.
\end{proof}

The idea of the proof of Theorem~\ref{thm:no-df-rl} (which can be found in Appendix~\ref{app:no-df-rl}) is a generalization of the proof of Lemma~\ref{lemma:dictators} that dictators are not robustly learnable. 
However, note that we construct a distribution whose support is all of $\X$.
It is possible to find two hypotheses $c_1$ and $c_2$ and create a distribution such that $c_1$ and $c_2$ will likely look identical on samples of size polynomial in $n$ but have robust risk $\Omega(1)$ with respect to one another. 
Since any hypothesis $h$ in $\set{0,1}^\X$ will disagree either with $c_1$ or $c_2$ on a given point $x$ if $c_1(x)\neq c_2(x)$, by choosing the target hypothesis $c$ at random from $c_1$ and $c_2$, we can guarantee that $h$ won't be robust against $c$ with positive probability.
Finally, note that an analogous argument can be made for a more general setting (for example in $\R^n$).

\section{Monotone Conjunctions}
\label{sec:mon-conj}

It turns out that we do not need recourse to ``bad'' distributions to show that very simple classes of functions are not efficiently robustly learnable.
As we demonstrate in this section, \textsf{MON-CONJ}, the class of monotone conjunctions, is not efficiently robustly learnable \emph{even under the uniform distribution} for robustness parameters that are superlogarithmic in the input dimension. 

\subsection{Non-Robust Learnability}
\label{sec:mon-conj-main}

The idea to show that \MonConj  is not efficiently robustly learnable is in the same vein as the proof of Theorem~\ref{thm:no-df-rl}.
We first start by proving the following lemma, which lower bounds the robust risk of two disjoint monotone conjunctions. 
\begin{lemma}
\label{lemma:bound-loss}
Under the uniform distribution, for any $n\in\N$, disjoint $c_1,c_2\in\MonConj$ of length $3\leq l\leq n/2$  on $\boolhc$ and robustness parameter $\rho\geq l/2$, we have that $\roblosse(c_1,c_2)$ is bounded below by a constant that can be made arbitrarily close to $\frac{1}{2}$ as $l$ gets larger. 
\end{lemma}
\begin{proof}
For a  hypothesis $c\in\MonConj$, let $I_c$ be the set of variables in $c$.
Let $c_1,c_2\in\C$ be as in the theorem statement.
Then the robust risk $\roblosse(c_1,c_2)$ is bounded below by 
\begin{equation*}
\Pr{x\sim D}{c_1(x)=0\wedge x\text{ has at least $l/2$ 1's in }I_{c_2}}= (1-2^{-l})/2\enspace.
\end{equation*}
\end{proof}

Now, the following lemma shows that if we choose the length of the conjunctions $c_1$ and $c_2$ to be super-logarithmic in $n$, then, for a sample of size polynomial in $n$, $c_1$ and $c_2$ will agree on $S$ with probability at least $1/2$.
The proof can be found in Appendix~\ref{app:concepts-agree}.

\begin{lemma}
\label{lemma:concepts-agree}
For any functions $l(n)=\omega(\log(n))$ and $m(n)=\poly(n)$, for any disjoint monotone conjunctions $c_1,c_2$ such that $|I_{c_1}|=|I_{c_2}|=l(n)$, there exists $n_0$ such that for all $n\geq n_0$, a sample $S$ of size $m(n)$ sampled i.i.d. from $D$ will have that $c_1(x)=c_2(x)=0$ for all $x\in S$ with probability at least $1/2$.
\end{lemma}

We are now ready to prove our main result of the section.

\begin{theorem}
\label{thm:mon-conj}
\MonConj is not efficiently $\rho$-robustly learnable for $\rho(n)=\omega(\log(n))$ under the uniform distribution.
\end{theorem}
\begin{proof}

Fix any algorithm $\A$ for learning \MonConj.
We will show that the expected robust risk between a randomly chosen target function and any hypothesis returned by $\A$ is bounded below by a constant.
Fix a function $\poly(\cdot,\cdot,\cdot,\cdot,\cdot)$, and note that, since $\text{size}(c)$ and $\rho$ are both at most $n$, we can simply consider a function $\poly(\cdot,\cdot,\cdot)$ in the variables $1/\epsilon,$ and $1/\delta,n$ instead.
Let $\delta=1/2$, and fix a function $l(n)=\omega(\log(n))$ that satisfies $l(n)\leq n/2$, and let $\rho(n)=l(n)/2$ ($n$ is not yet fixed).
Let $n_0$ be as in Lemma~\ref{lemma:concepts-agree}, where $m(n)$ is the fixed sample complexity function.
Then Equation~(\ref{eqn:sample-size}) holds for all $n\geq n_0$.

Now, let $D$ be the uniform distribution on $\boolhc$ for $n\geq \max(n_0,3)$, and choose $c_1$, $c_2$ as in Lemma~\ref{lemma:bound-loss}.
Note that $\roblosse(c_1,c_2)>\frac{5}{12}$ by the choice of $n$.
Pick the target function $c$ uniformly at random between $c_1$ and $c_2$, and label $S\sim D^m$ with $c$, where $m=\poly(1/\epsilon,1/\delta,n)$.
By Lemma~\ref{lemma:concepts-agree}, $c_1$ and $c_2$ agree with the labeling of $S$ (which implies that all the points have label $0$) with probability at least~$\frac{1}{2}$ over the choice of $S$. 

Define the following three events for $S\sim D^m$: 
\begin{align*}
&\E:\;{c_1}_{|S}={c_2}_{|S}\;,\enspace
\E_{c_1}:\;c=c_1\;,\enspace
\E_{c_2}:\;c=c_2\enspace.
\end{align*}

Then, by Lemmas~\ref{lemma:concepts-agree} and~\ref{lemma:robloss-triangle}, 

\begin{align*}
\eval{c,S}{\roblosse(\A(S),c)}
&\geq\Pr{c,S}{\E}\eval{c,S}{\roblosse(\A(S),c)\;|\;\E}\\
&>\frac{1}{2}\left(
\Pr{c,S}{\E_{c_1}}\eval{S}{\roblosse(\A(S),c)\;|\;\E\cap\E_{c_1}}
+\Pr{c,S}{\E_{c_2}}\eval {S}{\roblosse(\A(S),c)\;|\;\E\cap \E_{c_2}}
\right)\\
&=\frac{1}{4}\;\eval{S}{\roblosse(\A(S),c_1)+\roblosse(\A(S),c_2)\;|\;\E}\\
&\geq \frac{1}{4}\;\eval{S}{\roblosse(c_2,c_1)}\\
&>0.1
\enspace.
\end{align*}

\end{proof}

\subsection{Robust Learnability Against a Logarithmically-Bounded Adversary}
\label{sec:rob-learning-unif}
The argument showing the non-robust learnability of \MonConj under the uniform distribution in the previous section cannot be carried through if the conjunction lengths are logarithmic in the input dimension, or if the robustness parameter is small compared to that target conjunction's length.
In both cases, we show that it is possible to efficiently robustly learn these conjunctions if the class of distributions is $\alpha$-$\log$-Lipschitz, i.e. there exists  a universal constant $\alpha\geq1$ such that for all $n\in \N$, all distributions $D$ on $\boolhc$ and for all input points $x,x'\in \boolhc$, if $d_H(x,x')=1$, then $|\log(D(x))-\log(D(x'))|\leq\log(\alpha)$ (see Appendix~\ref{app:log-lipschitz} for further details and useful facts).

\begin{theorem}
\label{thm:mon-conj-rob}
  Let $\mathcal{D}=\{\mathcal{D}_n\}_{n\in\mathbb{N}}$, where $\mathcal{D}_n$ is a set of
     $\alpha$-$\log$-Lipschitz distributions on $\{0,1\}^n$ for all $n\in\mathbb{N}$.   Then the class of
    monotone conjunctions is $\rho$-robustly learnable with respect to
    $\mathcal{D}$ for robustness function $\rho(n)=O(\log n)$.
\end{theorem}

The proof can be found in Appendix~\ref{app:mon-conj}.
This combined with Theorem~\ref{thm:mon-conj-rob} shows that $\rho(n)=\log (n)$ is essentially the threshold for efficient robust learnability of the class \MonConj.


\section{Computational Hardness of Robust Learning}
\label{sec:comp-hardness}

In this section, we establish that the computational hardness of PAC-learning a concept class $\C$ with respect to a distribution class $\D$ implies the computational hardness of robustly learning a family of concept-distribution pairs from a related class $\C'$ and a restricted class of distributions $\D'$.
This is essentially a version of the main result of \cite{bubeck2018adversarial}, which used the constant-in-the-ball definition of robust risk. 
Our proof also uses the \cite{bubeck2018adversarial} trick of encoding a point's label in the input for the robust learning problem.
Interestingly, our proof does not rely on any assumption other than the existence of a hard learning problem in the PAC framework and is valid under \emph{both} Definitions~\ref{def:loss-correct} and~\ref{def:loss-constant}  of robust risk.

\paragraph{Construction of $\C'$.}
Suppose we are given $\C=\set{\C_n}_{n\in\N}$ and $\D=\set{\D_n}_{n\in\N}$ with $\C_n$ and $\D_n$ defined on $\X_n=\boolhc$.
Given $k\in\N$, we define the family of concept and distribution pairs $\{(c',D')\}_{D'\in\D'_{c'},c'\in\C'}$, where $\C'=\{\Cprime\}_{k,n\in\N}$ 
 on $\Xprime=\boolhcprime$ as follows.
Let $\maj_k:\Xprime\rightarrow\X_n$ be the function that returns the majority vote on each subsequent block of $k$ bits, and ignores the last bit.
We define  $\Cprime = \set{c\circ \maj_{2k+1}\;|\; c\in\C_n}$.
Let $\varphi_k:\X_n\rightarrow\Xprime$ be defined as
\begin{align*}
&\varphi_k(x):= \underbrace{x_1\dots x_1x_2\dots x_{d-1} x_d\dots x_d}_{2k+1\text{ copies of each }x_i}c(x)
\;,\quad\varphi_k(S):= \set{\varphi_k(x_i)\;|\; x_i\in S}
\enspace,
\end{align*}
for $x=x_1x_2\dots x_d\in \X$ and $S\subseteq \X$.
For a concept $c\in\C_n$, each $D\in\D_n$ induces a distribution $D'\in\D'_{c'}$, where $c'=c\circ \maj_{2k+1}$ and $D'(z)=D(x)$ if $z=\varphi_k(x)$, and $D'(z)=0$ otherwise.

As shown below, this set up allows us to see that any algorithm for learning $\C_n$ with respect to $\D_n$ yields an algorithm for learning the pairs $\{(c',D')\}_{D'\in\D'_{c'},c'\in\C'}$. 
However, any \emph{robust} learning algorithm cannot solely rely on the last bit of the input, as it could be flipped by an adversary.
Then, this algorithm can be used  to PAC-learn $\C_n$. 
This establishes the equivalence of the computational difficulty between PAC-learning $\C_n$ with respect to $\Dprime$ and robustly learning $\{(c',D')\}_{D'\in\D'_{c'},c'\in\Cprime}$. 
As mentioned earlier, we can still efficiently PAC-learn the pairs $\{(c',D')\}_{D'\in\D'_{c'},c'\in\C'}$  simply by always outputting a hypothesis that returns the last bit of the input.

\begin{theorem}
\label{thm:comp-hardness}
For any concept class $\C_n$, family of distributions $\D_n$ over $\boolhc$ and $k\in\N$, there exists a concept class $\Cprime$ and a family of distributions $\Dprime$ over $\boolhcprime$ such that \emph{efficient} $k$-robust learnability of  the concept-distribution pairs $\{(c',D')\}_{D'\in\D'_{c'},c'\in\Cprime}$   and either of the robust risk functions $\mathsf{R}^C_k$ or $\mathsf{R}^E_k$ implies \emph{efficient} PAC-learnability of $\C_n$ with respect to $\D_n$.
\end{theorem}

Before proving the above result, let us first prove the following proposition.

\begin{proposition}
\label{prop:sample-compl}
The concept-distribution pairs $\{(c',D')\}_{D'\in\D'_{c'},c'\in\Cprime}$  can be $k$-robustly learned using $O\left(\frac{1}{\epsilon}\left(\log\abs{\C_n}+\log\frac{1}{\delta}\right)\right)$ examples.
\end{proposition}

\begin{proof}
First note that, since $\C_n$ is finite,  we can use PAC-learning sample bounds for the realizable setting (see for example \cite{mohri2012foundations}) to get that the sample complexity of learning $\C_n$ is $O\left(\frac{1}{\epsilon}(\log\abs{\C_n}+\log\frac{1}{\delta})\right)$.
Now, if we have PAC-learned $\C_n$ with respect to $\D_n$, and $h$ is the hypothesis returned on a sample labeled according to a target concept $c\in\C_n$, we can compose it with the function $\maj_k$ to get a hypothesis $h'$ for which any perturbation of at most $k$ bits of $x'\sim D'$ (where $D'$ is the distribution induced by the target concept $c$ and distribution $D$) will not change $h'(x')$. 
Thus, we also have $k$-robustly learned  $\Cprime$.
\end{proof}

\begin{remark}
The sample complexity in Proposition~\ref{prop:sample-compl} is independent of $k$, and so the construction of the class $\C'$ on $\X'$ allows the adversary to modify $\frac{1}{2n}$ fraction of the bits. 
There are ways to make the adversary more powerful and keep the sample complexity unchanged. 
Indeed, the fraction of the bits the adversary can flip can be increased by using error correction codes.
For example, BCH codes \cite{bose1960class,hocquenghem1959codes} would allow us to obtain an input space $\X'$ of dimension $n+k\log n$ where the adversary can flip $\frac{k}{n+k\log n}$ bits.
\end{remark}

We are now ready to prove the main result of this section.

\begin{proof}[Proof of Theorem~\ref{thm:comp-hardness}]
Given $\C_n$ and $\D$, let $\Cprime$ and $\{\D'_{c'}\}_{c'\in\Cprime}$ be constructed as above. 
Suppose that it is hard to PAC-learn $\C_n$ with respect to the distribution family $\D_n$.
Suppose that we are given an algorithm $\A'$ to $k$-robustly learn $\{(c',D')\}_{D'\in\D'_{c'},c'\in\Cprime}$ and a sample complexity $m$.

Let $\epsilon,\delta>0$ be arbitrary and $c\in \C_n$ be an arbitrary target concept and let $c'\in\Cprime$ be such that $c'=c\circ \maj_{2k+1}$. 
Let $D\in\D_n$ be a distribution on $\X_n$, and let $D'\in\D'_{c'}$ be its induced distribution on $\Xprime$. 
A PAC-learning algorithm for $\C_n$ is as follows. 
Draw a sample $S\sim D^m$ and let $S'=\varphi_k(S)$.
Note that this simulates a sample $S'\sim D'^m$, and that $c'$ will give the same label to all points in the $\rho$-ball centred at $x'$ for any $x'$ in the support of $D'$.

Since $\A'$ $k$-robustly learns the concept-distribution pairs $\{(c',D')\}_{D'\in\D'_{c'},c'\in\Cprime}$, with probability at least $1-\delta$ over $S'$, for any $x\sim D$, we have that $h'$ will be wrong on  $\varphi_k(x)$ (where the last bit is random) with probability at most $\epsilon$. 
So by outputting $h=h'\circ\varphi_k$, we have an algorithm to PAC-learn $\C_n$ with respect to the distribution family $\D_n$. 
\end{proof}

\section{Conclusion}

We have studied robust learnability from a computational learning theory perspective and have shown that efficient robust learning can be hard -- even in very natural and apparently straightforward settings.
We  have moreover given a tight characterization of the strength of an adversary to prevent robust learning of monotone conjunctions under certain distributional assumptions. 
An interesting avenue for future work is to see whether this result can be generalised to other classes of functions.
Finally, we have provided a simpler proof of the previously established result of the computational hardness of robust learning.

In the light of our results, it seems to us that more thought needs to be put
into what we want out of robust learning in terms of computational efficiency
and sample complexity, which will inform our choice of risk functions. Indeed,
at first glance, robust learning definitions that have appeared in prior work
seem in many ways natural and reasonable; however, their inadequacies surface
when viewed under the lens of computational learning theory. Given our negative
results in the context of the current robustness models, one may surmise that
requiring a classifier to be correct in an entire ball near a point is asking
for too much. Under such a requirement, we can only solve ``easy problems''
with strong distributional assumptions.  Nevertheless, it may still be of
interest to study these notions of robust learning in different learning
models, for example where one has access to membership queries.



\bibliographystyle{plainnat}
\bibliography{references}
\newpage

\section*{Appendix}
\appendix

\section{Learning Theory Basics}
\subsection{The PAC framework}

We study the problem of robust classification.
This is a generalization of  standard classification tasks, which are defined on an input space $\X_n$ of dimension $n$ and finite output space $\Y$. 
Common examples of input spaces are $\{0,1\}^n$, $[0,1]^n$, and $\R^n$.
We focus on \emph{binary classification} in the \emph{realizable setting}, where $\Y=\{0,1\}$, and we get access to a sample $S=\{(x_i,y_i)\}_{i=1}^m$ where the $x_i$'s are drawn i.i.d. from an unknown underlying distribution $D$, and there exists $c:\X\rightarrow\Y$ such that $y_i=c(x_i)$, namely, there exists a \emph{target concept} that has labeled the sample.
In the PAC framework \cite{valiant1984theory}, our goal is to find a function $h$ that approximates $c$ with high probability over the training sample.
This means we are allowing a small chance of having a sample that is not representative of the distribution. 
As we require our confidence to increase, we require more data. 
PAC learning is formally defined for  \emph{concept classes} $\C_n\subseteq\{0,1\}^{\X_n}$ as follows.

\begin{definition}[PAC Learning]
Let $\C_n$ be a concept class over $\X_n$ and let $\C=\bigcup_{n\in\N}\C_n$.
We say that $\C$ is \emph{PAC learnable using hypothesis class $\H$} and sample complexity function $p(\cdot,\cdot,\cdot)$ if there exists an algorithm $\A$ that satisfies the following:
for all $n\in\N$, for every $c\in\C_n$, for every $D$ over $\X_n$, for every $0<\epsilon<1/2$ and $0<\delta<1/2$, if whenever $\A$ is given access to $m\geq p(n,1/\epsilon,1/\delta)$ examples drawn i.i.d. from $D$ and labeled with $c$, $\A$ outputs $h\in\H$ such that with probability at least $1-\delta$, 
\begin{equation*}
\Pr{x\sim D}{c(x)\neq h(x)}\leq \epsilon\enspace.
\end{equation*}
We say that $\C$ is statistically efficiently PAC learnable if $p$ is polynomial in $n,1/\epsilon$ and $1/\delta$, and computationally efficiently PAC learnable if $\A$ runs in polynomial time in $n,1/\epsilon$ and $1/\delta$.
\end{definition}

PAC learning is \emph{distribution-free}, in the sense that no assumptions are made about the distribution from which the data comes from.
The setting where $\C=\H$ is called \emph{proper learning}, and \emph{improper learning} otherwise.

\subsection{Monotone Conjunctions}

A conjunction $c$ over $\{0,1\}^n$ can be represented a set of literals $l_1,\dots,l_k$, where, for $x\in\X_n$, $c(x)=\bigwedge_{i=1}^k l_i$. 
For example, $c(x)=x_1\wedge\bar{x_2}\wedge{x_5}$ is a conjunction.
Monotone conjunctions are the subclass of conjunctions where negations are not allowed, i.e. all literals are of the form $l_i=x_j$ for some $j\in[n]$.

The standard PAC learning algorithm to learn monotone conjunctions is as follows.
We start with the hypothesis $h(x)=\bigwedge_{i\in I_h} x_i$, where $I_h=[n]$. 
For each example $x$ in $S$, we remove $i$ from $I_h$ if $c(x)=1$ and $x_i=0$. 

When one has access to membership queries, one can easily exactly learn monotone conjunctions over the whole input space: we start with the instance where all bits are 1 (which is always a positive example), and we can test whether each variable is in the target conjunction by setting the corresponding bit to 0 and requesting the label.

We refer the reader to \cite{mohri2012foundations} for an in-depth introduction to machine learning theory.

\subsection{Log-Lipschitz Distributions}
\label{app:log-lipschitz}

\begin{definition}
A distribution  $D$ on $\boolhc$ is said to be $\alpha$-$\log$-Lipschitz if 
for all input points $x,x'\in \boolhc$, if $d_H(x,x')=1$, then $|\log(D(x))-\log(D(x'))|\leq\log(\alpha)$.
\end{definition}

The intuition behind $\log$-Lipschitz distributions is that points 
that are close to each other must not have frequencies that greatly 
differ from each other. Note that, by definition, $D(x)>0$ for all inputs $x$.
Moreover,  the uniform distribution is $\log$-Lipschitz
with parameter $\alpha=1$. Another example of $\log$-Lipschitz 
distributions is the class of product distributions where the probability
of drawing a $0$ (or equivalently a $1$) at index $i$ is in the interval 
$\left[\frac{1}{1+\alpha},\frac{\alpha}{1+\alpha}\right]$. 
Log-Lipschitz distributions have been studied in \cite{awasthi2013learning}, 
and its variants in \cite{feldman2012data,koltun2007approximately}.

Log-Lipschitz distributions have the following useful properties, which 
we will often refer to in our proofs.
\begin{lemma}
\label{lemma:log-lips-facts}
Let $D$ be an $\alpha$-$\log$-Lipschitz distribution over $\boolhc$. 
Then the following hold:
\begin{enumerate}
\item\label{test} For $b\in\{0,1\}$, $\frac{1}{1+\alpha}\leq \Pr{x\sim D}{x_i=b}\leq\frac{\alpha}{1+\alpha}$.
\item For any $S\subseteq[n]$, the marginal distribution $D_{\bar{S}}$ is $\alpha$-$\log$-Lipschitz, where $D_{\bar{S}}(y)=\sum_{y'\in\{0,1\}^S} D(yy')$.
\item For any $S\subseteq[n]$ and for any property $\pi_S$ that only depends on variables $x_S$, the marginal with respect to $\bar{S}$ of the conditional distribution $(D|\pi_S)_{\bar{S}}$ is $\alpha$-$\log$-Lipschitz.
\item For any $S\subseteq[n]$ and $b_S\in\{0,1\}^S$,  we have that $\left(\frac{1}{1+\alpha}\right)^{|S|}\leq \Pr{x\sim D}{x_i=b}\leq\left(\frac{\alpha}{1+\alpha}\right)^{|S|}$.
\end{enumerate}
\end{lemma}

\begin{proof}

To prove (1), fix $i\in[n]$ and $b\in\{0,1\}$ and denote by $x^{\oplus i}$ the result of flipping the $i$-th bit of $x$. Note that

\begin{equation*}
\Pr{x\sim D}{x_i=b}
=\sum_{\substack{z\in\boolhc :\\ z_i=b}} D(z)
=\sum_{\substack{z\in\boolhc :\\ z_i=b}} \frac{D(z)}{D(z^{\oplus i})}D(z^{\oplus i})
\leq\alpha\sum_{\substack{z\in\boolhc :\\ z_i=b}} D(z^{\oplus i})
=\alpha\Pr{x\sim D}{x_i\neq b}
\enspace.
\end{equation*}
The result follows from solving for $\Pr{x\sim D}{x_i=b}$.

Without loss of generality, let $\bar{S}=\{1,\dots,k\}$ for some $k\leq n$. 
Let $x,x'\in \{0,1\}^{\bar{S}}$ with $d_H(x,x')=~1$.

To prove (2), let $D_{\bar{S}}$ be the marginal distribution. Then,
\begin{equation*}
D_{\bar{S}}(x)
=\sum_{y\in\{0,1\}^S}D(xy)
=\sum_{y\in\{0,1\}^S}\frac{D(xy)}{D(x'y)}D(x'y)
\leq \alpha \sum_{y\in\{0,1\}^S} D(x'y)
= \alpha D_{\bar{S}}(x')
\enspace.
\end{equation*}

To prove (3), denote by $X_{\pi_S}$ the set of points in $\{0,1\}^S$ satisfying property $\pi_S$, and by $xX_{\pi_S}$ the set of inputs of the form $xy$, where $y\in X_{\pi_S}$.  
By a slight abuse of notation, let $D(X_{\pi_S})$ be the probability of drawing a point in $\boolhc$ that satisfies $\pi_S$.
Then,
\begin{equation*}
D(xX_{\pi_S})
= \sum_{y\in X_{\pi_S}} D(xy)
= \sum_{y\in X_{\pi_S}} \frac{D(xy)}{D(x'y)} D(x'y)
\leq  \alpha \sum_{y\in X_{\pi_S}} D(x'y )
= \alpha D(x'X_{\pi_S})
\enspace.
\end{equation*}
We can use the above and show that
\begin{equation*}
(D|\pi_S)_{\bar{S}} (x)
= \frac{D(xX_{\pi_S}) }{D(x'X_{\pi_S})} \frac{D(x'X_{\pi_S})}{D(X_{\pi_S})}
\leq \alpha (D|\pi_S)_{\bar{S}} (x')\enspace.
\end{equation*}


Finally, (4) is a corollary of (1)--(3).

\end{proof}

\section{Discussion on the Relationship between Robust and Zero-Risk Learning}
\label{app:zero-risk-learning}

We saw that, for both robust risks $\roblossc$ and $\roblosse$, zero-risk learning does not necessarily imply robust learning. 
Moreover, as shown in Section~\ref{sec:no-df-rl}, efficient distribution-free robust learning is not possible even in the realizable setting.
What can be said if we have access to a robust learning algorithm for a specific distribution on the boolean hypercube?
We will show that distribution-dependent robust learning implies zero-risk learning for both robust risk definitions, under certain conditions on the measure of balls in the support of the distribution.
Let us start with Definition~\ref{def:loss-correct}, where we require the hypothesis to be exact in the $\rho$-balls around a point.

\begin{proposition} 
\label{prop:robust-loss-threshold-1}
For any probability measure $\mu$ on $\boolhc$, robustness parameter $\rho$ and concepts $h,c$, there exists $\epsilon>0$ such that if $\roblosse(h,c)<\epsilon$ then $h(x)=c(x)$ for any $x\in\X$ such that $\mu(B_\rho(x))>0$.
In particular, one has that $h$ and $c$ agree on the support of $\mu$.
\end{proposition}
\begin{proof}
Suppose there exists $x^*\in\X$ with $\mu(B_\rho(x^*))>0$ such that $h(x^*)\neq c(x^*)$. 
Then for any $z\in B_\rho(x^*)$, we have that $\roblosse(h,c,z)$, the robust risk of $h$ with respect to $c$ at point $z$, is 1.
Let $\tilde{\X}:=\set{x\in\X:\mu\left(B_\rho(x)\right)>0}$, and  $\epsilon=\min_{x\in\tilde{\X}}\mu(B_\rho(x))$. 
We have that
\begin{align*}
\roblosse(h,c)
&\geq\sum_{z\in B_\rho(x^*)}\mu(\set{z})\ell_\rho^R(h,c,z)
=\mu(B_\rho(x^*))
\geq\epsilon
\enspace.
\end{align*} 
\end{proof}

\begin{corollary}
\label{cor:robust-to-exact-1}
For any fixed distribution $D$, robust learning with respect to $D$ implies zero-risk learning with respect to $D$ for any robustness parameter as long as $\epsilon$ in Proposition~\ref{prop:robust-loss-threshold-1} satisfies $\epsilon^{-1}=\poly(n)$.
\end{corollary}

\begin{proof}
Fix a distribution $D\in\D$ on $\X$. 
Suppose that we have a $\rho$-robust learning algorithm $\A_\F^R(D)$ for $\F$, namely for all $\epsilon,\delta,\rho>0$, for all $c\in\F$, if $\A_\F^R(D)$ has access to a sample $S$ of size $m\geq \poly(\frac{1}{\epsilon},\frac{1}{\delta},\text{size}(c),n)$, it returns $f\in\F$ such that 
\begin{equation}
\Pr{S\sim D^m}{\ell_\rho^R(f,c)<\epsilon}\geq 1-\delta
\enspace.
\end{equation}

By Proposition~\ref{prop:robust-loss-threshold-1}, we can choose $\epsilon$ such that $\roblosse(h,c)<\epsilon$ implies that $h(x)=c(x)$ for any $x\in\X$ such that $\mu(B_\rho(x))>0$. 
Note that this $\epsilon$ depends on $D$, $\rho$ and $n$.
So we have that 

\begin{equation}
\Pr{x\sim D}{f(x)\neq c(x)}=0
\enspace,
\end{equation}
with probability at least $1-\delta$ over the training sample $S$, whose size remains polynomial in $\frac{1}{\delta}$ and $n$ by the proposition assumptions.
\end{proof}

\begin{remark}
The assumption on  $\epsilon$ in Corollary~\ref{cor:robust-to-exact-1} is necessary to use the robust learning algorithm as a black box: in Section~\ref{sec:rob-learning-unif}, we work under a well-behaved class of distributions that includes the uniform distribution and show that, for long enough monotone conjunctions and small enough robustness parameter (with respect to the conjunction length), efficient robust learning is possible.
However, we cannot exactly learn these monotone conjunctions.
In the uniform distribution setting, the $\rho$-balls all have the same probability mass and $\epsilon^{-1}$ is essentially superpolynomial in $n$. 
\end{remark}

To show the same result  for $\roblossc$, where the hypothesis is constant in a ball, we can use the exact same reasoning as in Corollary~\ref{cor:robust-to-exact-1}, except that we need to show the analogue of Proposition~\ref{prop:robust-loss-threshold-1} for this setting.

\begin{proposition}
For any probability measure $\mu$ on $\boolhc$ and for any concepts $h,c$, there exists $\epsilon>0$ such that if $\roblossc(h,c)<\epsilon$ then $h$ and $c$ agree on the support of $\mu$.
\end{proposition}

\begin{proof}
Fix $h,c,D$ and let $\epsilon=\min_{x\in\supp(\mu)}\mu(\set{x})$.
Suppose there exists $x^*\in\supp(\mu)$ and $z\in B_\rho(x^*)$ such that $c(x^*)\neq h(z)$. Then
\begin{equation*}
\roblossc(h,c)
=\Pr{x\sim \mu}{\exists z\in B_\rho(x)\;.\; c(x)\neq h(z)}
\geq\epsilon
\enspace.
\end{equation*}
\end{proof}

\section{Proofs from Section~\ref{sec:no-df-rl}}
\label{app:no-df-rl}

\begin{proof}[Proof of Theorem~\ref{thm:no-df-rl}]
First, if $\C$ is trivial, we need at most one example to identify the target function.

For the other direction, suppose that $\C$ is non-trivial.
We first start by fixing any learning algorithm and polynomial sample complexity function $m$. 
Let $\eta=\frac{1}{2^{\omega(\log n)}}$, $0<\delta<\frac{1}{2}$, and note that for any constant $a>0$,
\begin{equation*}
\underset{n\rightarrow\infty}{\lim}\;
n^a\log(1-\eta)^{-1}=0
\enspace,
\end{equation*}
and so any polynomial in $n$ is $o\left(\left(\log(1/(1-\eta))\right)^{-1}\right)$.
Then it is possible to choose $n_0$ such that for all $n\geq n_0$, 
\begin{equation}
\label{eqn:sample-size-ub}
m\leq \frac{\log(1/\delta)}{2n\log(1-\eta)^{-1}}
\enspace.
\end{equation}

Since $\C$ is non-trivial, we can choose concepts $c_1,c_2\in \C_{n}$ and points $x,x'\in\boolhc$ such that $c_1$ and $c_2$ agree on $x$ but disagree on $x'$.
This implies that there exists a point $z\in\boolhc$ such that (i)~$c_1(z)=c_2(z)$ and (ii)~it suffices to change \emph{only one bit} in $I:=I_{c_1}\cup I_{c_2}$ to cause $c_1$ to disagree on $z$ and its perturbation.
Let $D$ be such that
\begin{align*}
&\Pr{x\sim D}{x_i=z_i}=
\begin{cases}
	1-\eta &\quad\text{if }i\in I\\
	\frac{1}{2}&\quad\text{otherwise}
\end{cases}
\enspace.
\end{align*}

Draw a sample $S\sim D^m$ and label it according to $c\sim U(c_1,c_2)$. 
Then,
\begin{equation}
\label{eqn:same-label-sample-2}
\Pr{S\sim D^m}{\forall x\in S\quad c_1(x)=c_2(x)}
\geq\left(1-\eta\right)^{m| I |}
\enspace.
\end{equation}
Bounding the RHS below by $\delta>0$, we get that, as long as 
\begin{equation}
\label{eqn:sample-size-2}
m\leq \frac{\log(1/\delta)}{| I |\log(1-\eta)^{-1}}
\enspace,
\end{equation}
(\ref{eqn:same-label-sample-2}) holds with probability at least $\delta$. 
But this is true as Equation~(\ref{eqn:sample-size-ub}) holds as well.
However, if $x=z$, then it suffices to flip one bit of $x$ to get $x'$ such that $c_1(x')\neq c_2(x')$. 
Then,
\begin{equation}
\roblosse(c_1,c_2)\geq\Pr{x\sim D}{x_{I}=z_{I}}=\left(1-\eta\right)^{| I |}
\enspace.
\end{equation}
The constraints on $\eta$ and the fact that $| I |\leq n$ are sufficient to guarantee that the RHS is $\Omega(1)$.
Let $\alpha>0$ be a constant such that  $\roblosse(c_1,c_2)\geq\alpha$.

We can use the same reasoning as in Lemma~\ref{lemma:robloss-triangle} to argue that, for any $h\in\set{0,1}^\X$, 
\begin{equation*}
 \mathsf{R}^E_1(c_1,h)+ \mathsf{R}^E_1(c_2,h)
\geq  \mathsf{R}^E_1(c_1,c_2) \enspace.
\end{equation*}
Finally, we can show that
\begin{equation*}
\underset{c\sim U(c_1,c_2)}{\mathbb{E}}\eval{S\sim D^m}{ \mathsf{R}^R_1(h,c)} \geq \alpha\delta/2 ,
\end{equation*}
hence there exists a target $c$ with expected robust risk bounded below by a constant\footnote{For a more detailed reasoning, we refer the reader to the proof of Theorem~\ref{thm:mon-conj}, where we bound the expected value $\eval{c,S}{\roblosse(\A(S),c)}$ of the robust risk of a target chosen at uniformly random and the hypothesis outputted by a learning algorithm $\A$ on a sample $S$.}.
\end{proof}

\section{Proofs from Section~\ref{sec:mon-conj}}
\label{app:mon-conj}

\subsection{Proof of Lemma~\ref{lemma:concepts-agree}}
\label{app:concepts-agree}
\begin{proof}

We begin by bounding the probability that $c_1$ and $c_2$ agree on an i.i.d. sample of size $m$:
\begin{equation}
\label{eqn:zero-label-sample}
\Pr{S\sim D^m}{\forall x\in S  \cdot c_1(x)=c_2(x)=0}
=\left(1-\frac{1}{2^l}\right)^{2m}
\enspace.
\end{equation}
Bounding the RHS below by $1/2$, we get that, as long as 
\begin{equation}
\label{eqn:sample-size}
m\leq \frac{\log(2)}{2\log(2^l/(2^l-1))}
\enspace,
\end{equation}
(\ref{eqn:zero-label-sample}) holds with probability at least $1/2$.

Now, if $l=\omega(\log(n))$, then for a constant $a>0$,
\begin{equation*}
\underset{n\rightarrow\infty}{\lim}\;
n^a\log\left(\frac{2^l}{2^l-1}\right)=0
\enspace,
\end{equation*}
and so any polynomial in $n$ is $o\left(\left(\log\left(\frac{2^l}{2^l-1}\right)\right)^{-1}\right)$.
\end{proof}

\subsection{Proof of Theorem~\ref{thm:mon-conj-rob}}

\begin{proof}

  We show that the algorithm $\mathcal{A}$ for PAC-learning monotone
  conjunctions (see \cite{mohri2012foundations}, chapter 2) is a robust learner for an appropriate choice
  of sample size.  We start with the hypothesis $h(x)=\bigwedge_{i\in I_h} x_i$, 
  where $I_h=[n]$. For each example $x$ in $S$, we remove $i$ 
  from $I_h$ if $c(x)=1$ and $x_i=0$.

Let $\D$ be a class of $\alpha$-$\log$-Lipschitz distributions. 
 Let $n\in\mathbb{N}$ and $D\in \mathcal{D}_n$.  Suppose moreover
  that the target concept $c$ is a conjunction of $l$ variables.
  Fix $\varepsilon,\delta>0$. Let $\eta=\frac{1}{1+\alpha}$, 
  and note that by Lemma~\ref{lemma:log-lips-facts},
  for any $S\subseteq[n]$ and $b_S\in\{0,1\}^S$,  we have that 
  $\eta^{|S|}\leq \Pr{x\sim D}{x_i=b}\leq(1-\eta)^{|S|}$.

\paragraph{Claim 1.} 
If
$m \geq \left\lceil \frac{\log n-\log \delta}{\eta^{l+1}}
\right\rceil$ then given a sample $S \sim D^m$, algorithm
$\mathcal{A}$ outputs $c$ with probability at least $1-\delta$.

  \emph{Proof of Claim 1.} Fix $i \in \{1,\ldots,n\}$.  Algorithm
  $\mathcal{A}$ eliminates $i$ from the output hypothesis just in
  case there exists $x\in S$ with $x_{i}=0$ and $c(x)=1$.  Now we
  have $\Pr{x\sim D}{x_{i}=0 \wedge c(x)=1}\geq \eta^{l+1}$
  and hence
\[
\Pr{S\sim D}{\forall  x \in S \cdot i\text{ remains in }I_h}\leq 
(1-\eta^{l+1})^m 
 \leq  e^{-m\eta^{l+1}}
 =  \frac{\delta}{n} \, .
\]
The claim now follows from union bound over $i \in \{1,\ldots,n\}$.

\paragraph{Claim 2.}  If $l \geq \frac{8}{\eta^2}\log(\frac{1}{\varepsilon})$
and $\rho \leq \frac{\eta l}{2}$ then
$\Pr{x\sim D} {\exists z \in B_\rho(x) \cdot c(z)=1}\leq \varepsilon$.

\emph{Proof of Claim 2.}  
Define a random variable $Y=\sum_{i\in I_c} \mathbb{I}(x_i=1)$.  
We simulate $Y$ by the following process. 
Let $X_1,\dots,X_l$ be random variables taking value in $\{0,1\}$, and which may be dependent. 
Let $D_i$ be the marginal distribution on $X_i$ conditioned on $X_1,\dots,X_{i-1}$. 
This distribution is also $\alpha$-$\log$-Lipschitz by Lemma~\ref{lemma:log-lips-facts}, and hence,
\begin{equation}
\label{eqn:marg-bound}
\Pr{X_i\sim D_i}{X_i=1}\leq 1-\eta
\enspace.
\end{equation}



Since we are interested in the random variable $Y$ representing the number of 1's in $X_1,\dots,X_l$, 
we define the random variables $Z_1,\dots,Z_l$ as follows:
\begin{equation*}
Z_k = \left(\sum_{i=1}^k X_i\right)-k(1-\eta)\enspace.
\end{equation*} 
The sequence $Z_1, \dots, Z_l$ is a supermartingale with respect to $X_1,\dots,X_l$:
\begin{align*}
\eval{}{Z_{k+1}\given X_1,\dots,X_k}
&=\eval{}{Z_{k}+X_{k+1}'-(1-\eta)\given X'_1,\dots,X'_k}\\
&=Z_k+\Pr{}{X_{k+1}'=1\given X'_1,\dots,X'_k}-(1-\eta)\\
&\leq Z_k
\enspace. \tag{by (\ref{eqn:marg-bound})}
\end{align*}
Now, note that all $Z_k$'s satisfy $|Z_{k+1}-Z_k|\leq 1$, and that $Z_l=Y-l(1-\eta)$. 
We can thus apply the Azuma-Hoeffding (A.H.) Inequality to get 
\begin{align*}
\Pr{}{Y\geq l-\rho}
&\leq \Pr{}{Y\geq l(1-\eta)+\sqrt{2\ln(1/\varepsilon)l}}	\\
&=\Pr{}{Z_l-Z_0\geq \sqrt{2\ln(1/\varepsilon)l}}	\\
&\leq \exp\left(-\frac{\sqrt{2\ln(1/\varepsilon)l}^2}{2l}\right)				\tag{A.H.}\\
&=\varepsilon
\enspace,
\end{align*}
where the first inequality holds from the given bounds on $l$ and $\rho$:
\begin{align*}
l-\rho &=(1-\eta)l + \frac{\eta l}{2} 
               + \frac{\eta l}{2} - \rho \\
          & \geq (1-\eta) l + \frac{\eta l}{2} 
           \tag{since $\rho \leq \frac{\eta l}{2}$}\\
          & \geq  (1-\eta) l + \sqrt{2\log(1/\varepsilon) l} \enspace.
            \tag{since $l \geq \frac{8}{\eta^2}\log(\frac{1}{\varepsilon})$}
\end{align*}
This completes the proof of Claim 2.

We now combine Claims 1 and 2 to prove the theorem.  Define
$l_0 := \max(\frac{2}{\eta}\log n,
\frac{8}{\eta^2}\log(\frac{1}{\varepsilon}))$.  Define
$m:=\left\lceil \frac{\log n-\log \delta}{\eta^{l_0+1}}
\right\rceil$.  Note that $m$ is polynomial in $n$, $\delta$,
$\varepsilon$.  

Let $h$ denote the output of algorithm $\mathcal{A}$ given a sample
$S\sim D^m$.  We consider two cases.  If $l \leq l_0$ then, by
Claim 1, $h=c$ (and hence the robust risk is $0$) with probability at
least $1-\delta$.  If $l_0 \leq l$ then, since $\rho=\log n$, we
have $l \geq \frac{8}{\eta^2}\log(\frac{1}{\varepsilon})$ and
$\rho \leq \frac{\eta l}{2}$ and so we can apply Claim 2.  By Claim
2 we have
\[ \roblosse(h,c) \leq \Pr{x\sim D}{\exists z \in B_\rho(x) \cdot c(z)=1} \leq \varepsilon \]

\end{proof}

\end{document}